\documentclass{article}

    \PassOptionsToPackage{numbers, compress}{natbib}

 \usepackage[main, final]{neurips_2026}

\usepackage[utf8]{inputenc} 
\usepackage[T1]{fontenc}    
\usepackage{hyperref}       
\usepackage{url}            
\usepackage{booktabs}       
\usepackage{amsfonts}       
\usepackage{nicefrac}       
\usepackage{microtype}      
\usepackage{xcolor}         

\usepackage{amsmath}
\usepackage{amssymb}
\usepackage{mathtools}
\usepackage{amsthm}
\usepackage{color}%
\usepackage{enumitem}
\usepackage{multirow} 
\usepackage{tcolorbox}
\usepackage{pifont}
\usepackage{graphicx}   

\usepackage{algorithm}
\usepackage{algorithmic}

\usepackage{makecell}

\usepackage{wrapfig}
\usepackage{caption}
\usepackage{subcaption}

\usepackage{amsmath,amsfonts,bm}

\def\eqref#1{(\ref{#1})}

\def\1{\bm{1}}

\def\rvc{{\mathbf{c}}}

\def\rvg{{\mathbf{g}}}

\def\rvm{{\mathbf{m}}}

\def\rvx{{\mathbf{x}}}

\def\rvz{{\mathbf{z}}}

\DeclareMathAlphabet{\mathsfit}{\encodingdefault}{\sfdefault}{m}{sl}
\SetMathAlphabet{\mathsfit}{bold}{\encodingdefault}{\sfdefault}{bx}{n}

\def\gG{{\mathcal{G}}}

\def\gM{{\mathcal{M}}}

\def\gS{{\mathcal{S}}}

\def\gV{{\mathcal{V}}}

\newcommand{\E}{\mathbb{E}}

\DeclareMathOperator*{\argmax}{arg\,max}

\theoremstyle{plain}
\newtheorem{theorem}{Theorem}[section]
\newtheorem{proposition}[theorem]{Proposition}

\theoremstyle{definition}

\theoremstyle{remark}

\newcommand{\Cat}{{\rm{Cat}}}

\title{Self-Rewarding Sequential Monte Carlo for Masked Diffusion Language Models}

\author{%
  Ziwei Luo \\
  Uppsala University\\
  Sweden \\
  \And
  Ziqi Jin \\
  Nanyang Technological University \\
  Singapore \\
  \AND
  Lei Wang\thanks{Corresponding author} \\
  Apodex \\
  Singapore \\
  \And
  Lidong Bing \\
  Apodex \\
  Singapore \\
  \And
  Thomas B. Sch{\"o}n \\
  Uppsala University \\
  Sweden \\
}

\begin{document}

\maketitle

\begin{abstract}
  This work presents self-rewarding sequential Monte Carlo (SMC), an inference-time scaling algorithm enabling effective sampling of masked diffusion language models (MDLMs). Our algorithm stems from the observation that most existing MDLMs rely on a confidence-based sampling strategy, where only tokens with the highest prediction confidence are preserved at each step. This can restrict generation to a noise-sensitive, greedy decoding paradigm, often limiting trajectory diversity. We address this problem by launching multiple interacting diffusion processes in parallel, referred to as \emph{particles}, for trajectory exploration. Importantly, we introduce the trajectory-level confidence as a self-rewarding signal for assigning particle importance weights. During sampling, particles are iteratively weighted and resampled to systematically steer generation towards globally confident, high-quality samples. Our self-rewarding SMC is verified on various masked diffusion language models and benchmarks, achieving significant improvement without extra training or reward guidance, while effectively converting parallel inference capacity into improved sampling quality. Our code is available at \url{https://github.com/Algolzw/self-rewarding-smc}.
\end{abstract}

\section{Introduction}

Generative modeling with diffusion models~\citep{sohl2015deep,ho2020denoising} has led to remarkable advances across a wide range of applications, including image generation~\citep{dhariwal2021diffusion,rombach2022high,peebles2023scalable,dou2024diffusion}, text-to-image generation~\citep{saharia2022photorealistic,zhang2023adding,gu2022vector}, and video synthesis~\citep{ho2022video,blattmann2023stable}. More recently, diffusion models have further shown strong potential for discrete data generation, particularly for text~\citep{wu2023practical,sahoo2024simple,arriola2025block,nie2025large,ye2025dream}, by modeling a forward masking process and iteratively predicting masked tokens at inference~\citep{lou2024discrete,chang2022maskgit}. Despite this progress, existing masked diffusion language models (MDLMs) adopt a greedy, confidence-based sampling strategy~\cite{sahoo2024simple}, in which only tokens with the highest prediction probability are preserved at each step, leading to myopic trajectory exploration and suboptimal generation performance.

Inference-time scaling methods have been proposed to improve MDLMs' sample diversity and quality without modifying pretrained models~\citep{dang2025inference,singhal2025a,ou2025inference,lee2025debiasing}. Typically, they leverage human preference guidance, such as promoting fluency, enforcing structured format, or controlling toxicity~\citep{dathathri2019plug,rafailov2023direct,loula2025syntactic}, to tilt the generation trajectories towards high-reward target distributions~\citep{dang2025inference,uehara2025inference,kim2025test}. While their results are impressive, such methods rely on external reward signals that are often task-specific and require hand-crafted tuning, which restricts their applicability to general MDLM-based generation.

In this work, we revisit the confidence-based sampling of masked diffusion and propose a \emph{self-rewarding} sequential Monte Carlo (SMC) algorithm for inference-time scaling on general tasks. Specifically, we maintain a set of interacting diffusion processes, referred to as \emph{particles}, to explore multiple trajectories in parallel. These particles still follow the low-confidence remasking strategy of MDLMs but will be resampled based on their trajectory-level confidence, which serves as an implicit reward signal to assign importance weights to each particle. Conceptually, our algorithm runs multiple generation processes, periodically duplicating promising candidates and discarding less confident ones based on their accumulated likelihood. The resulting self-rewarding SMC enables principled trajectory exploration and steers the sampling process towards stable, globally confident, and high-quality outputs, without extra reward models or task-specific design as guidance.

Our method is verified on various masked diffusion language models including MDLM~\cite{sahoo2024simple} and BD3-LMs~\cite{arriola2025block}, and diffusion large language models (dLLMs) including LLaDA 1.5~\cite{zhu2025llada} and Dream-7B~\cite{ye2025dream}. The results show that the proposed self-rewarding SMC consistently improves the baseline models across multiple benchmarks. 

In summary, we make the following contributions:
\begin{itemize}
    \item We propose a general, self-rewarding sequential Monte Carlo algorithm for masked diffusion language models. The proposed method improves sampling without extra training or reward signals.
    \item We unify the sampling and remasking strategy of MDLMs from the probabilistic perspective, and theoretically show that the trajectory-level confidence is naturally a self-rewarding signal for SMC.
    \item Extensive experiments and analysis demonstrate our self-rewarding SMC improves sample quality on different pretrained models and benchmarks.
\end{itemize}

\section{Background}
\label{sec:background}

\paragraph{Notation} 
We consider $\rvx = \{ x_1, \dots, x_L \}$ as a sequence of $L$ tokens, where each token $x_\ell = \rvx(\ell)$ is a `one-hot' column vector with $V$ categories in the space $\gV \triangleq \{ x \in \{ 0, 1 \}^V : \sum_{i=1}^V x^{(i)} = 1 \} \subset \Delta^V$ for the simplex $\Delta^V$. We let the $V^{\text{th}}$ category denote a special $\mathtt{[MASK]}$ token, where $\rvm \in \gV$ is its one-hot vector. Moreover, we define $\Cat( \cdot ; p)$ as a categorical distribution with probability $p \in \Delta^V$.

\subsection{Masked Diffusion Models}
\label{subsec:mdm}

Given prior distribution $\Cat(\cdot; \rvm)$, masked diffusion models (MDMs)~\citep{austin2021structured,ou2025inference,shi2024simplified,sahoo2024simple} are characterized by parametric models $p_\theta$ trained to reverse a forward masking process for new data sampling from a full masked sequence. For a finite-time process, we let $T$ be the number of diffusion steps and $t(i)=\frac{i}{T} \in [0,1]$ be the continuous time. The marginal distribution of $\rvx_{t(i)}$ conditioned on the target data $\rvx_0$ is 
\begin{equation}
    p(\rvx_t \mid \rvx_0) = \Cat(\rvx_t; \, \alpha_t \rvx_0 + (1 - \alpha_t)\rvm),
    \label{eq:mdm_forward}
\end{equation}
where $\alpha_t$ denotes a monotonically decreasing schedule satisfying $\alpha_0 \approx 1$ and $\alpha_1 \approx 0$. Let $s(i)=\frac{i-1}{T}$ be the time step directly preceding $t(i)$, the time-reversal conditional posterior for all $\mathtt{[MASK]}$ tokens i.e., $\rvx_t = \rvm$, can be obtained by
\begin{equation}
    p(\rvx_s \mid \rvx_t, \rvx_0) = \Cat(\rvx_s; \, \frac{\alpha_s - \alpha_t}{1 - \alpha_t} \rvx_0 + \frac{1 - \alpha_s}{1 - \alpha_t} \rvx_t).
    \label{eq:mdm_backward}
\end{equation}
Notably, unmasked tokens i.e., $\rvx_t \neq \rvm$, remain unchanged in the reverse process. Since $\rvx_0$ is not available during inference, the reverse unmasking process is parametrized as $p_\theta(\rvx_s \mid \rvx_t, \hat{\rvx}_0)$, where $\hat{\rvx}_0 \sim \Cat(\rvx_0;p_\theta(\rvx_t))$ is sampled from the model predictive distribution. To simplify the notation, we only consider sampling from the diffusion reverse process and denote $t \coloneqq t(i)$ as the discrete time step throughout the work.

\begin{figure}[ht]
    \centering
    \includegraphics[width=1.\linewidth]{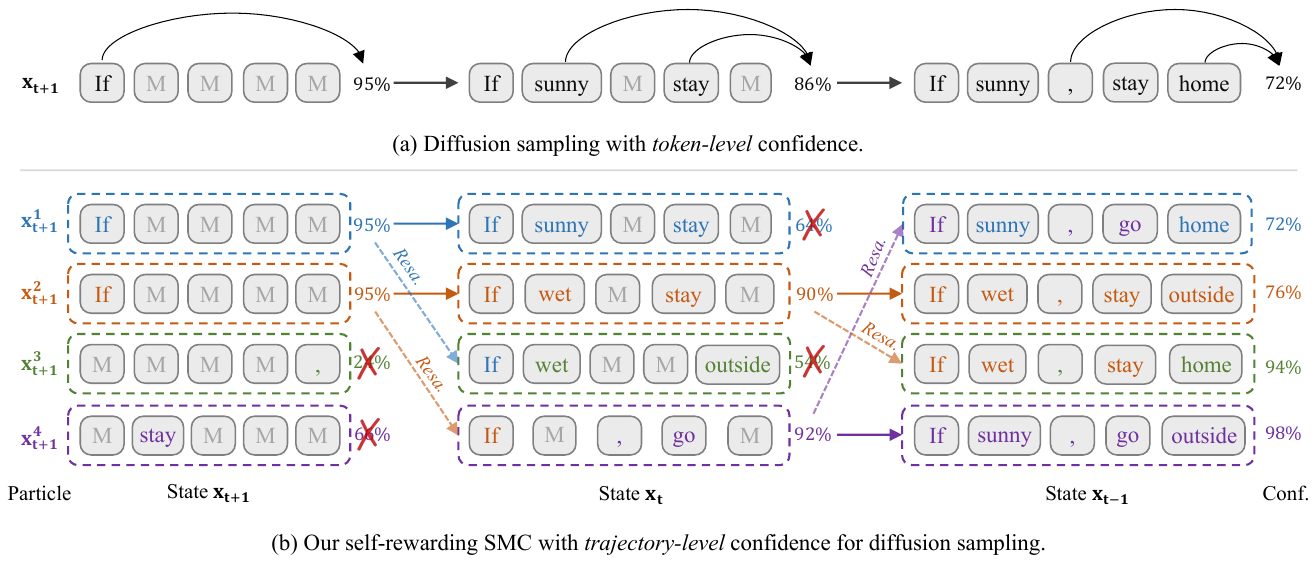}
    \caption{Illustrative example of text generation using (a) masked diffusion models and (b) our self-rewarding SMC framework. Here, `M' represents  $\mathtt{[MASK]}$ tokens and `\textit{Resa.}' denotes resampling. Our method maintains multiple diffusion processes, called \emph{particles}, to explore trajectories in parallel. 
    }
    \label{fig:framework}
\end{figure}

\subsection{Importance Sampling and Sequential Monte Carlo}
\label{subsec:is_smc}

Assume we want to estimate expectations under trajectory target distribution $\pi(\rvx_{t:T})$, such as $\E_\pi[f(\rvx_{t:T})]$, where $f(\cdot)$ is a test function. Sampling from $\pi$ is generally intractable.

\textbf{Importance sampling (IS)}~\cite{robert1999monte} introduces a proposal distribution $q(\rvx_{t:T})$ that allows the expectation to be rewritten as
\begin{equation}
        \E_\pi[f(\rvx_{t:T})] = \E_q[ \frac{\tilde{\pi}(\rvx_{t:T})}{q(\rvx_{t:T})} f(\rvx_{t:T})] \approx \frac{1}{N} \sum_{i=1}^N w_t^i f(\rvx_{t:T}^i),
        \label{eq:is_approx}
\end{equation}
where $\rvx_{t:T}^i \sim q(\rvx_{t:T})$ and $\tilde{\pi}$ is the unnormalized density. Moreover, $w_t^i = {\tilde{W}_t^i}/{\sum_{j=1}^N \tilde{W}_t^j}$ is the normalized importance weight, where $\tilde{W}_t^i = {\tilde{\pi}(\rvx_{t:T}^i)}/{q(\rvx_{t:T}^i)}$ gauges the discrepancy between the target distribution and the proposal distribution. While conceptually simple, importance sampling often suffers from unfavorably high variance. 

\textbf{Sequential Monte Carlo (SMC)}~\cite{del2006sequential,cardoso2023monte} improves upon IS by introducing 
a sequence of intermediate unnormalized path measures $\tilde{\pi}_t(\rvx_{t:T})$, whose terminal distribution coincides with the desired trajectory target distribution. SMC incorporates resampling and sequential weighting techniques across the trajectory, thereby reducing variance in practice. We begin by defining the incremental importance weights as
\begin{equation}
    \tilde{w}_{t-1}(\rvx_{t-1:T}) = \frac{\tilde{\pi}_{t-1}(\rvx_{t-1:T})}{\tilde{\pi}_t(\rvx_{t:T}) \, q_{t-1}(\rvx_{t-1} \mid \rvx_t)},
    \label{eq:smc_iiweights}
\end{equation}
where $q_{t-1}(\rvx_{t-1} \mid \rvx_t)$ is a Markovian sequential proposal operating in reverse time (see Appendix~\ref{app-subsec:smc_iiweight} for more details). During sampling, we initialize a set of $N$ particles $\rvx_T^i \sim q_T(\rvx_T)$, each representing a trajectory distribution, with weights $w_T^i = \tilde{\pi}_T(\rvx_T^i) / q_T(\rvx_T^i)$. At each iteration e.g., from $t$ to $t-1$, SMC takes the following three steps: 
\begin{enumerate}[
  topsep=2pt,
  itemsep=1pt,
  parsep=0pt,
  partopsep=0pt
]
    \item[i)] \textit{Resample}: resample ancestor $\{ \rvx_t^i \}_{i=1}^N$ according to the weights $\{ w_t^i \}_{i=1}^N$;
    \item[ii)] \textit{Propagate}: sample new particles from proposal distribution $\rvx_{t-1}^i \sim q_{t-1}(\rvx_{t-1} \mid \rvx_t^i)$;
    \item[iii)] \textit{Re-weight}: compute and accumulate Eq.~\eqref{eq:smc_iiweights}, and normalize weight $w_{t-1}^i = \frac{\tilde{w}_{t-1}^i}{\sum_{j=1}^N \tilde{w}_{t-1}^j}$.
\end{enumerate}
The resulting collection of weighted particles provides an asymptotically consistent approximation of the trajectory target distribution.

\section{Method}
\label{sec:method}

\subsection{Reformulate the Sampling of MDMs}

We consider sampling from a pretrained masked diffusion model $p_\theta(\rvx_t)$. Given a mask set $\gM_t \triangleq \{\ell: \rvx_t(\ell)=\mathtt{[MASK]}\}$ at time $t$, recall that the learned posterior $p_\theta(\rvx_{t-1} \mid \rvx_t, \hat{\rvx}_0)$ in Eq.~\eqref{eq:mdm_backward} is only applied for mask tokens i.e.,
\begin{equation}
    \rvx_{t-1}(j) \sim p_\theta(\rvx_{t-1} (j) \mid \rvx_t, \hat{\rvx}_0), \quad j \in \gM_t,
\end{equation}
where $\hat{\rvx}_0 \sim \Cat(\rvx_0;p_\theta(\rvx_t))$ is sampled from the model predictive distribution. For each token $\rvx_{t-1}(j)$, we directly define its confidence as the model probability on $j$, as
\begin{equation}
    \rvc_t (j) \coloneqq p_\theta\big(\hat{\rvx}_0(j) \mid \rvx_t \big), \quad j \in \gM_t.
    \label{eq:mdm_confidence}
\end{equation}
At each iteration, MDMs update a subset $\gS_t \subseteq \gM_t$ following a predefined policy, such as the \emph{low-confidence remasking} strategy~\cite{nie2025large}, to ensure an iterative unmasking process from $\rvx_T$ to $\rvx_0$. 

\textbf{Low-confidence remasking} has been widely used in diffusion large language models as an efficient strategy for sequence generation. Typically, by defining a schedule $\rho_t$ to specify the number of tokens to be unmasked at step $t$, we introduce the following policy:
\begin{equation}
    \gS_t^{\text{top-k}} = \{\, j\in \gM_t:\; \text{Top-}\rho_t \{ \rvc_t{(j)} \} \,\},
    \label{eq:lcr_topk}
\end{equation}
which indicates only the highest probability tokens are preserved at each step. Note that the schedule $\rho_t$ can be a scalar function of $t$, or be more flexible by explicitly defining $\rho_t$ as a confidence threshold~\cite{wu2025fast}, such that
\begin{equation}
    \gS_t^{\text{thr}} = \{\, j\in \gM_t:\; \rvc_t{(j)} \ge \rho_t \,\},
    \label{eq:lcr_threshold}
\end{equation}
which enables faster sampling while preserving performance when the model is confident in its predictions.
In summary, the reverse transition distribution of each token $\rvx_t (j)$ can be formulated by
\begin{equation}
    p_\theta(\rvx_{t-1} (j) \mid \rvx_t, \hat{\rvx}_0) =
    \begin{cases}
        p_\theta(\rvx_{t-1} (j) \mid \rvx_t, \hat{\rvx}_0), & j\in S_t,\\
        \Cat(\rvx_{t-1} (j); \rvm), & j\in M_t\setminus S_t,\\
        \Cat(\rvx_{t-1} (j); \rvx_t), & j\notin M_t,
    \end{cases}
    \label{eq:mdm_token_transition}
\end{equation}
where $\Cat(\rvx_{t-1} (j); \rvx_t)$ is a point-mass categorical distribution concentrated at $\rvx_t (j)$. Accordingly, the reverse transition kernel over the full sequence $\rvx_t$ is
\begin{equation}
    K_t(\rvx_t, \rvx_{t-1}) = \prod_{j=1}^L p_\theta(\rvx_{t-1} (j) \mid \rvx_t, \hat{\rvx}_0).
    \label{eq:mdm_transition_kernel}
\end{equation}
This transition kernel deterministically preserves unmasked tokens, remasks low-confidence tokens, and samples newly accepted tokens according to the model prediction.

One problem of sampling from Eq.~\eqref{eq:mdm_token_transition} is that only step-wise confidence i.e., $\rvc_t$ in Eq.~\eqref{eq:mdm_confidence}, is utilized for remasking. This often biases generation towards locally optimal tokens, inducing noise-sensitive, myopic exploration of the sequence trajectory, as illustrated in Figure~\ref{fig:framework}.

\subsection{Self-Rewarding Sequential Monte Carlo}

Assume that we have~$N$ diffusion sampling processes, called particles, to generate $N$ sequences in parallel. To tilt sampling towards globally confident sequence generation, we define a Feynman--Kac model~\cite{del2004feynman} with potential 
$G_{t-1}(\rvx_t,\rvx_{t-1}) = \prod_{j\in S_t} p_\theta(\rvx_{t-1}(j) \mid \rvx_t)$,
which is the joint probability of accepted tokens within the set~$\gS_t$. Intuitively, the potential denotes how confident the model is in the tokens at step~$t$, performing as a self-rewarding signal for SMC update. In addition, we define the intermediate unnormalized path measures $\tilde{\pi}_t(\rvx_{t:T})$ to satisfy the following recursion:
\begin{equation}
    \tilde{\pi}_{t-1}(\rvx_{t-1:T}) = \tilde{\pi}_t(\rvx_{t:T})\, K_t(\rvx_t,\rvx_{t-1})\, G_{t-1}(\rvx_t,\rvx_{t-1}),
    \label{eq:fk_recursion}
\end{equation}
where $K_t(\rvx_t,\rvx_{t-1})$ is the reverse diffusion transition kernel in Eq.~\eqref{eq:mdm_transition_kernel}. Importantly, the potential $G$ is used only for particle reweighting and resampling. Thus this recursion does not replace the diffusion sampler with a temperature-modified transition kernel. Recall that SMC defines an incremental importance weight for each particle (see Eq.~\eqref{eq:smc_iiweights}). By letting its proposal equal the transition kernel in Eq.~\eqref{eq:fk_recursion}, we obtain the following result:
\begin{proposition}
\label{prop:confidence_weights}
Given a pretrained diffusion model $p_\theta$, let $\{\tilde\pi_t(\rvx_{t:T})\}_{t=0}^T$ denote the unnormalized path measures defined by the recursion in Eq.~\eqref{eq:fk_recursion}. If the sequential proposal in SMC is chosen to be the diffusion transition kernel, i.e., $q_{t-1}(\rvx_{t-1} \mid \rvx_t) = K_t(\rvx_t,\rvx_{t-1})$,
then the incremental importance weights at step $t-1$ are given by
\begin{equation}
    \tilde{w}_{t-1}(\rvx_{t-1:T}) = \prod_{j\in S_t} \rvc_t(j),
    \label{eq:conf_iiweights}
\end{equation}
where $\rvc_t (j) \coloneqq p_\theta\big(\hat{\rvx}_0(j) \mid \rvx_t \big)$ is the token confidence assigned by the model predictive distribution and $\gS_t$ denotes the selected mask subset to be updated at step $t$.
\end{proposition}
The proof is provided in Appendix~\ref{app-subsec:proof_prop31}.  Under our SMC framework, Proposition~\ref{prop:confidence_weights} shows that the incremental weight takes the form of a trajectory-level confidence score, accumulating the confidences of accepted tokens across sampling steps until the sequence is fully unmasked. Specifically, when the diffusion transition kernel is used as the proposal and the accepted-token likelihood is used as the potential, the resulting SMC weight analytically reduces to a simple product of token confidences.

\subsection{Practical Sampling with Self-Rewarding SMC}

\begin{wrapfigure}{R}{0.57\textwidth}
    \vspace{-2.2em}
    \begin{minipage}{0.57\textwidth}
    \begin{algorithm}[H]
    \caption{Self-Rewarding SMC (SR-SMC)} 
    \label{alg:srsmc}
    \small
    \begin{algorithmic}[1]
        \vspace{.01in}
        \REQUIRE Pretrained diffusion model $p_\theta$, sampling steps $T$, number of particles $N$, remasking policy $\textsc{Select}(\cdot)$.
        \ENSURE Generated sequence $\hat{\rvx}_0$.
        
        \STATE \textbf{Initialize} $N$ particles i.e., sequences $\{ \rvx_T^i \}_{i=1}^N$ with all tokens set to $\mathtt{[MASK]}$, and weights $w_T^i = 1/N$ for all $i$.
        \FOR{$t=T, \dotsc, 1$}
          \STATE \textbf{Resample} $\{ \rvx_t^i \}_{i=1}^N$ according to weights $\{ w_t^i \}_{i=1}^N$.
          \STATE \textbf{Propagate} with mask set $\gM_t^i$ for all $i$: 
          \STATE \quad i) \; sample $\hat{\rvx}_0^i \sim p_\theta(\rvx_t^i)$ and compute confidence $\rvc_t^i$.
          \STATE \quad ii) \ select update set $\gS_t^i \leftarrow \textsc{Select}(\rvc_t^i,\gM_t^i)$.
          \STATE \quad iii) sample $\rvx_{t-1}^i \sim K_t(\rvx_t^i,\rvx_{t-1}^i)$ using Eq.~\eqref{eq:mdm_token_transition}.
          \STATE \textbf{Re-weight} by computing $\tilde{w}_{t-1}^i$ using Eq.~\eqref{eq:conf_iiweights}.
        \ENDFOR
        \STATE \textbf{return} $\hat{\rvx}_0 \leftarrow \rvx_0^{i^\star}$ where $i^\star = \argmax_i \tilde{w}_0^i$.
    \vspace{.0in}
    \end{algorithmic}
    \end{algorithm}
    \end{minipage}
    \vspace{-1.em}
  \end{wrapfigure}
  
We now describe how Eq.~\eqref{eq:conf_iiweights} is incorporated into the sampling procedure of masked diffusion models via sequential Monte Carlo. As shown in Figure~\ref{fig:framework} and Algorithm~\ref{alg:srsmc}, we begin by initializing $N$ particles that are fully masked sequences, with uniform weights. At each step, we perform \emph{resample}, \emph{propagate}, and \emph{re-weight} as a standard SMC described in Sec.~\ref{subsec:is_smc}. Specifically, the diffusion sampling with local confidence-based remasking is performed during propagation, denoted by the transition kernel in Eq.~\eqref{eq:mdm_token_transition} and Eq.~\eqref{eq:mdm_transition_kernel}. Then each particle is reweighted according to Eq.~\eqref{eq:conf_iiweights}, which forms a trajectory-level confidence score for resampling. The final output is selected from the resulting particle set with the maximum weight. In addition, we use effective sample size (ESS)~\cite{doucet2001introduction} and Gumbel-Max trick~\cite{huijben2022review} to further improve the sample efficiency.

\textbf{Adaptive resampling.} In practice, resampling at every diffusion step might be unnecessary in particular when the variance of weights $w_t$ is low. We therefore adopt an adaptive resampling strategy based on the effective sample size, which is defined as 
$\mathrm{ESS} = 1/ \sum_{i=1}^N (w_t^i)^2$.
We follow a common practical setting~\cite{doucet2001introduction} to let resample be triggered only when $\mathrm{ESS}$ falls below $N/2$, which indicates significant weight degeneracy.

\textbf{Gumbel-Max sampling.} Following~\citet{zheng2024masked}, we employ the Gumbel-Max trick to sample discrete tokens from a controlled categorical
distribution for masked diffusion models. Particularly, given logits $\rvz_\ell$ over the vocabulary, token sampling is performed as
\begin{equation}
    \rvx = \argmax_\ell \big( {\rvz_\ell}/{\tau} + \rvg_\ell \big), \quad \rvg_\ell \sim \gG(0,1),
    \label{eq:gumbel_max}
\end{equation}
where $\tau$ is a temperature and $\gG$ denotes the Gumbel distribution. This Gumbel-Max trick approximates sampling from $\Cat(\cdot;\mathrm{softmax}(\rvz))$. Note that the temperature $\tau = 0$ represents deterministic $\mathtt{argmax}$ decoding and $\tau = 1$ recovers the standard categorical sampling.

\section{Experiments}
\label{sec:experiment}

Our self-rewarding SMC is evaluated across multiple pretrained models and benchmarks to demonstrate its ability to improve sampling of diffusion language models.

\subsection{Experimental Setup}

\textbf{Pretrained models.} We investigate two kinds of pretrained models: 1) masked diffusion language models, including MDLM~\cite{sahoo2024simple} and BD3-LMs~\cite{arriola2025block} pretrained on the OpenWebText (OWT) dataset ~\cite{Gokaslan2019OpenWeb} for sample quality evaluation; and 2) diffusion large language models including LLaDA~1.5~\cite{zhu2025llada} and Dream-7B~\cite{ye2025dream}. All of them are pretrained and finetuned on $>$2T tokens for general task evaluations. Both settings adopt a semi-autoregressive (Semi-AR) generation structure for higher-quality sequence generation. In addition, all models except MDLM employ a block-wise generation policy for a more efficient sampling.

\textbf{Implementation.} For all experiments, SR-SMC and the corresponding baselines use the same hyperparameters, including temperature, top-$p$ nucleus sampling threshold, sequence length, and block size when applicable. We do not perform method-specific hyperparameter tuning for SR-SMC or any baseline. The default number of particles is set to 4. Moreover, we use temperature $\tau=1$ and identical decoding settings for all comparisons. For dLLMs experiments, we follow~\citet{wu2025fast} to enable KV cache and parallel decoding (i.e., with a threshold-based policy, see Eq.~\eqref{eq:lcr_threshold}) for inference acceleration. We test MDLM and BD3-LMs on a single NVIDIA H200 GPU and evaluate all dLLMs experiments on 8 NVIDIA A800 GPUs with a single batch size.


\begin{table}[t]
    \centering
    \caption{Performance of Self-Rewarding SMC (SR-SMC) on Diffusion Large Language Models. Results compare baselines versus SR-SMC variants using block decoding with block size = 32 with KV cache and parallel decoding enabled~\cite{wu2025fast}.}
    \vspace{0.05in}
    \label{tab:benchmark_results}
    \resizebox{.76\linewidth}{!}{
    \begin{tabular}{lc| cc cc}
        \toprule
        Benchmark & Length & LLaDA-1.5 & \textit{w/ SR-SMC} & Dream-7B & \textit{w/ SR-SMC} \\
        \midrule
        GSM8K & 256 & 79.8 & \textbf{80.7} {\scriptsize\textcolor{gray}{(+0.9)}} & 76.2 & \textbf{78.0} {\scriptsize\textcolor{gray}{(+1.8)}} \\
        {\footnotesize (5-shot)} & 512 & 80.4 & \textbf{82.0} {\scriptsize\textcolor{gray}{(+1.6)}} & 78.0 & \textbf{78.0} {\scriptsize\textcolor{gray}{(+0.0)}} \\
        \midrule
        MATH & 256 & 38.2 & \textbf{41.8} {\scriptsize\textcolor{gray}{(+3.6)}} & 41.6 & \textbf{45.2} {\scriptsize\textcolor{gray}{(+3.6)}} \\
        {\footnotesize (4-shot)} & 512 & 41.4 & \textbf{45.4} {\scriptsize\textcolor{gray}{(+4.0)}} & 44.6 & \textbf{47.6} {\scriptsize\textcolor{gray}{(+3.0)}} \\
        \midrule
        HumanEval & 256 & 38.7 & \textbf{41.5} {\scriptsize\textcolor{gray}{(+2.8)}} & 47.9 & \textbf{53.7} {\scriptsize\textcolor{gray}{(+5.8)}} \\
        {\footnotesize (0-shot)} & 512 & 35.4 & \textbf{41.5} {\scriptsize\textcolor{gray}{(+6.1)}} & 45.1 & \textbf{53.7} {\scriptsize\textcolor{gray}{(+8.6)}} \\
        \midrule
        MBPP & 256 & 40.4 & \textbf{44.2} {\scriptsize\textcolor{gray}{(+3.8)}} & 42.0 & \textbf{48.6} {\scriptsize\textcolor{gray}{(+6.6)}} \\
        {\footnotesize (3-shot)} & 512 & 40.2 & \textbf{43.2} {\scriptsize\textcolor{gray}{(+3.0)}} & 39.6 & \textbf{46.4} {\scriptsize\textcolor{gray}{(+6.8)}} \\
        \midrule
        \multirow{2}{*}{Average}
        & 256 & 49.3 & \textbf{52.1} {\scriptsize\textcolor{gray}{(+2.8)}} & 51.9 & \textbf{56.4} {\scriptsize\textcolor{gray}{(+4.5)}} \\
        & 512 & 49.4 & \textbf{53.0} {\scriptsize\textcolor{gray}{(+3.6)}} & 51.8 & \textbf{56.4} {\scriptsize\textcolor{gray}{(+4.6)}} \\
        \bottomrule
    \end{tabular}}
\vspace{-0.1in}
\end{table}

\subsection{Sample Quality Evaluation}
\label{sec:results_gen_ppl}

\begin{wrapfigure}{r}{0.4\textwidth}
\setlength{\abovecaptionskip}{0.05in}
\setlength{\belowcaptionskip}{-0.2in}
\captionof{table}{Generative perplexity (Gen. PPL; $\downarrow$) comparison of 300 samples with sequence lengths $L=1024,2048$. We do not perform hyperparameter tuning.}
    \label{tab:gen_ppl}
    \centering
    \resizebox{1.\linewidth}{!}{
    \begin{tabular}{lcc}
      \toprule
      Model & $L=1024$ & $L=2048$ \\
      \midrule
      AR & 14.1 & 13.2 \\
      \midrule 
      \multicolumn{3}{l}{\textbf{Diffusion}} \\
      SEDD & 52.0 & -- \\
      MDLM & 46.8 & 41.3 \\
      MDLM \textit{w/ SR-SMC} & 25.8 & 25.9 \\
      \midrule
      \multicolumn{3}{l}{\textbf{Block Diffusion} (with block size $L'$)} \\
      SSD-LM \\ 
      \hspace{4.1em} $L'=25$ & 37.2 & 35.3 \\
      BD3-LMs \\ 
      \hspace{4.1em} $L'=16$ & 33.4 & 31.5 \\
      \hspace{4.1em} $L'=8$ & 30.4 & 28.2 \\
      \hspace{4.1em} $L'=4$ & 25.7 & 23.6 \\ 
      BD3-LMs \textit{w/ SR-SMC}  \\
      \hspace{4.1em} $L'=16$ & 21.1 & 20.2 \\
      \hspace{4.1em} $L'=8$ & 18.9 & 17.3 \\
      \hspace{4.1em} $L'=4$ & \textbf{16.1} & \textbf{15.1} \\
      \bottomrule
    \end{tabular}}
\vspace{-0.15in}
\end{wrapfigure}

We first evaluate SR-SMC on pretrained masked diffusion language models, including MDLM~\cite{sahoo2024simple} and BD3-LMs~\cite{arriola2025block}, for open-ended text generation. We also compare against autoregressive (AR), SEDD~\cite{lou2024discrete}, and SSD-LM~\cite{han2023ssd} baselines. As shown in Table~\ref{tab:gen_ppl} and Figure~\ref{app-fig:gen_ppl} (in Appendix), we generate sequences with lengths $L=1024$ and $L=2048$, and evaluate sample quality using generative perplexity under GPT2-Large. 
One can observe that SR-SMC consistently improves both MDLM and BD3-LMs. For MDLM, SR-SMC reduces Gen. PPL from $46.8$ to $25.8$ at $L=1024$, and from $41.3$ to $25.9$ at $L=2048$. For BD3-LMs, SR-SMC further improves all block-size variants, with the strongest setting $L'=4$ achieving $16.1$ and $15.1$ Gen. PPL at $L=1024$ and $L=2048$, respectively, approaching the AR baseline. These results indicate that SR-SMC effectively converts additional inference-time compute into improved text sample quality. We further report entropy results in Table~\ref{app-tab:gen_ppl_entropy}, showing that the improvement is not simply due to low-entropy collapse.

\subsection{Results on Diffusion Large Language Models}
\label{sec:results_dllms}

To further evaluate the effectiveness and generalizability of our SR-SMC, we conduct experiments on two representative diffusion large language models: LLaDA-1.5 \cite{zhu2025llada} and Dream-7B \cite{ye2025dream}. Following \citet{wu2025fast}, we evaluate these models across four challenging benchmarks: GSM8K and MATH for mathematical reasoning, and HumanEval and MBPP for code generation. The performance is measured using two different generation lengths ($L \in \{256, 512\}$) with a block size of 32.

The results are summarized in Table~\ref{tab:benchmark_results}. Overall, incorporating SR-SMC in sampling consistently improves the results across all benchmarks, model architectures, and generation lengths. Typically, our algorithm achieves average performance gains of \textbf{2.8+} and \textbf{4.5+} for LLaDA-1.5 and Dream-7B, respectively. This indicates a strong and consistent benefit of inference-time scaling through particle-based sampling. Moreover, the improvements on both mathematical and coding tasks suggest that SR-SMC generalizes well across different task domains. These gains are also sustained when generating much longer sequences, indicating that our trajectory-level resampling effectively mitigates error accumulation in diffusion-based generation. A qualitative comparison is shown in Figure~\ref{fig:case_study}. Collectively, these findings indicate that SR-SMC can serve as an effective and robust solution to computational scaling in dLLMs.

\begin{figure}[t]
    \centering
    \resizebox{0.9\linewidth}{!}{
    \begin{tcolorbox}[colback=gray!2, colframe=gray!50, arc=2pt, boxrule=0.5pt, title=\textbf{Qualitative Comparison: Mathematical Reasoning}]
        \small
        \textbf{Problem:} In the land of Ink, the money system is unique. One Trinket is equal to 4 Blinkets, and 3 Blinkets are equal to 7 Drinkets. In Trinkets, what is the value of 56 Drinkets? \\
        \textbf{Reference Answer:} 6
        
        \vspace{0.2cm}
        \begin{minipage}[t]{0.48\linewidth}
            \begin{tcolorbox}[colback=red!5, colframe=red!30, title=\textbf{Greedy Decoding ($\tau=0$)}, height=6.5cm]
                \scriptsize
                To find the value of 56 Drinkets in Trinkets, we need to convert all units to Trinkets. We know the following conversion rates: \\
                1. 1 Trinket = 4 Blinkets \\
                2. \textbf{3 Blinkets = 7 Blinkets} \textcolor{red}{(Local Consistency Error)} \\
                \\
                First, let's convert 3 Blinkets to Trinkets: \\
                $3 \text{ Blinkets} = \frac{3}{4} \text{ Trinkets}$. \\
                Next, let's convert 7 Blinkets to Trinkets: \\
                $7 \text{ Blinkets} = 7 \times \frac{3}{4} \text{ Trinkets} = 5.25 \text{ Trinkets}$. \\
                Now, let's convert 56 Drinkets to Trinkets: \\
                $56 \text{ Drinkets} = 56 \times 5.25 \text{ Trinkets} = 294 \text{ Trinkets}$. \\
                \\
                \textbf{Final Answer: \boxed{294}} \textcolor{red}{\ding{55}}
            \end{tcolorbox}
        \end{minipage}
        \hfill
        \begin{minipage}[t]{0.51\linewidth}
            \begin{tcolorbox}[colback=green!10, colframe=green!40!black!30, title=\textbf{w/ SR-SMC ($\tau=1, N=4$)}, height=6.5cm]
                \scriptsize
                To find the value of 56 Drinkets in Trinkets, we need to convert Drinkets to Blinkets first and then from Blinkets to Trinkets. \\
                \\
                1. Convert Drinkets to Blinkets: \\
                - We know that 3 Blinkets are equal to 7 Drinkets. \\
                - Therefore, 1 Drinket is equal to $\frac{3}{7}$ Blinkets. \\
                - So, 56 Drinkets is equal to $56 \times \frac{3}{7} = 24$ Blinkets. \\
                \\
                2. Convert Blinkets to Trinkets: \\
                - We know that 1 Trinket is equal to 4 Blinkets. \\
                - Therefore, 1 Blinket is equal to $\frac{1}{4}$ Trinkets. \\
                - So, 24 Blinkets is equal to $24 \times \frac{1}{4} = 6$ Trinkets. \\
                \\
                Therefore, the value of 56 Drinkets in Trinkets is \boxed{6} Trinkets. \\
                \\
                \textbf{Final Answer: \boxed{6}} \textcolor{green!70!black}{\ding{51}}
            \end{tcolorbox}
        \end{minipage}
    \end{tcolorbox}
    }
    \caption{A qualitative comparison of reasoning trajectories. Greedy decoding leads to a hallucinated identity (3 Blinkets = 7 Blinkets) that persists through the chain. SR-SMC utilizes trajectory-level confidence to explore multiple trajectories and successfully recovers the correct multi-step conversion.}
    \vspace{-0.3cm}
    \label{fig:case_study}
\end{figure}

\section{Discussion and Analysis}

\subsection{Scaling the Number of Particles}

\begin{wrapfigure}{r}{0.55\textwidth}
\vspace{-0.2in}
\centering
\setlength{\abovecaptionskip}{0.04in}
\setlength{\belowcaptionskip}{-0.0in}

\captionof{table}{Scaling the number of particles $N$ for LLaDA-1.5 and Dream-7B (with $L=256$) across benchmarks.}
\label{tab:particle_ablation_transpose}
\resizebox{\linewidth}{!}{
\begin{tabular}{ll|cccc}
    \toprule
    Model & Benchmark & \textbf{$N=1$} & \textbf{$N=2$} & \textbf{$N=3$} & \textbf{$N=4$} \\
    \midrule
    \multirow{5}{*}{LLaDA-1.5} 
    & GSM8K {\footnotesize (5-shot)} & 79.8 & \textbf{81.7} & 80.7 & 80.7 \\
    & MATH {\footnotesize (4-shot)} & 38.2 & 41.8 & 41.8 & 41.8 \\
    & HumanEval {\footnotesize (0-shot)} & 38.7 & 39.0 & 39.0 & \textbf{41.5} \\
    & MBPP {\footnotesize (3-shot)} & 40.4 & 44.2 & 44.2 & 44.2 \\
    \cmidrule{2-6}
    & Average & 49.3 & 51.7 & 51.4 & \textbf{52.1} \\
    \midrule
    \multirow{5}{*}{Dream-7B} 
    & GSM8K {\footnotesize (5-shot)} & 76.2 & 76.3 & 77.9 & \textbf{78.0} \\
    & MATH {\footnotesize (4-shot)} & 41.6 & 43.2 & \textbf{45.6} & 45.2 \\
    & HumanEval {\footnotesize (0-shot)} & 47.9 & 52.4 & \textbf{54.6} & 53.7 \\
    & MBPP {\footnotesize (3-shot)} & 42.0 & 43.6 & 44.2 & \textbf{48.6} \\
    \cmidrule{2-6}
    & Average & 51.9 & 53.9 & 55.6 & \textbf{56.4} \\
    \bottomrule
\end{tabular}}

\vspace{0.08in}

\captionof{table}{Scaling the number of particles for MDLM w/ SR-SMC on text generation with $L=1024$.}
\label{tab:scale_particles_mdlm}
\resizebox{0.82\linewidth}{!}{
\begin{tabular}{lcccc}
  \toprule
  \#Particles & $N=1$ & $N=4$ & $N=8$ & $N=16$ \\
  \midrule 
  Gen. PPL$\downarrow$ & 46.8 & 25.8 & 22.7 & 20.5 \\
  Entropy $\uparrow$ & 5.33 & 5.15 & 5.08 & 5.00 \\
  \bottomrule
\end{tabular}}

\vspace{-0.12in}
\end{wrapfigure}
As the number of particles N increases, we observe a clear and consistent improvement in performance across both large models (LLaDA-1.5 and Dream-7B) on all four benchmarks, as shown in Table~\ref{tab:particle_ablation_transpose} and Figure~\ref{fig:n_particles}. When $N=1$, the models reduce to standard parallel decoding, which serves as a lower bound on performance across all benchmarks. Scaling the particles to $N=2,3,4$ steadily improves performance (see Figure~\ref{fig:n_particles}), with the strongest gains typically achieved at $N=3$ or $N=4$. On average, scaling particles to~$N=4$ improves LLaDA-1.5 from $49.3$ to $52.1$ and Dream-7B from $51.9$ to $56.4$. The same trend is also observed on the text generation task in Table~\ref{tab:scale_particles_mdlm}. Since the text generation models are relatively small, we further scale the number of particles up to $N=16$, which continues to improve generative perplexity. These results show that increasing the number of particles effectively expands the search space over generation trajectories, allowing the model to recover from locally suboptimal token choices and accumulate higher trajectory-level confidence. Interestingly, even a modest increase to~$N=2$ yields significant gains over the baseline, highlighting that our SR-SMC as a practical and scalable inference-time method to masked diffusion language models.

\begin{wrapfigure}{r}{0.4\textwidth}
\setlength{\abovecaptionskip}{0.05in}
\setlength{\belowcaptionskip}{-0.2in}
\centering
\captionof{table}{Analysis of particle ``Overtake''. 
}
    \label{tab:overtake_analysis}
    \resizebox{.9\linewidth}{!}{
    \begin{tabular}{lcc}
        \toprule
        Benchmark  & \# Overtake & Percentage \\
        \midrule
        GSM8K     & 160 / 640 & 25.0\% \\
        MBPP     & 154 / 640 & 24.1\% \\
        MATH      & 197 / 640 & 30.8\% \\
        HumanEval  & 196 / 640 & 30.6\% \\
        \midrule
        \textbf{Average}                    & -   & \textbf{27.6\%} \\
        \bottomrule
    \end{tabular}}
    
\vspace{-0.15in}
\end{wrapfigure}

\subsection{Analysis of Particle Overtake}
To investigate the mechanical advantage of maintaining multiple interacting trajectories, we quantify the occurrence of \emph{overtake} events during the block-wise generation process. We define an overtake event as a case where the particle with the highest probability at the end of a block was not the dominant particle (i.e., highest probability) at the start of that same block. 

Table \ref{tab:overtake_analysis} analyzes 640 blocks per benchmark using LLaDA-1.5. We observe that overtake events occur in approximately 24\% to 31\% of all blocks across tasks. This empirical evidence indicates that the particles in SR-SMC are not merely performing idle exploration around greedy trajectories. Instead, the resampling mechanism preserves initially non-dominant but high-potential trajectories, allowing them to eventually outperform the local greedy choices that standard decoding methods would otherwise lock into. Notably, such behavior cannot be captured by single-trajectory decoding or purely greedy confidence-based remasking, which irreversibly commits to the locally dominant particle at the beginning of each block.

\begin{figure}[t]
\setlength{\abovecaptionskip}{0.05in}
\setlength{\belowcaptionskip}{-0.1in}
    \centering
    \includegraphics[width=1.\linewidth]{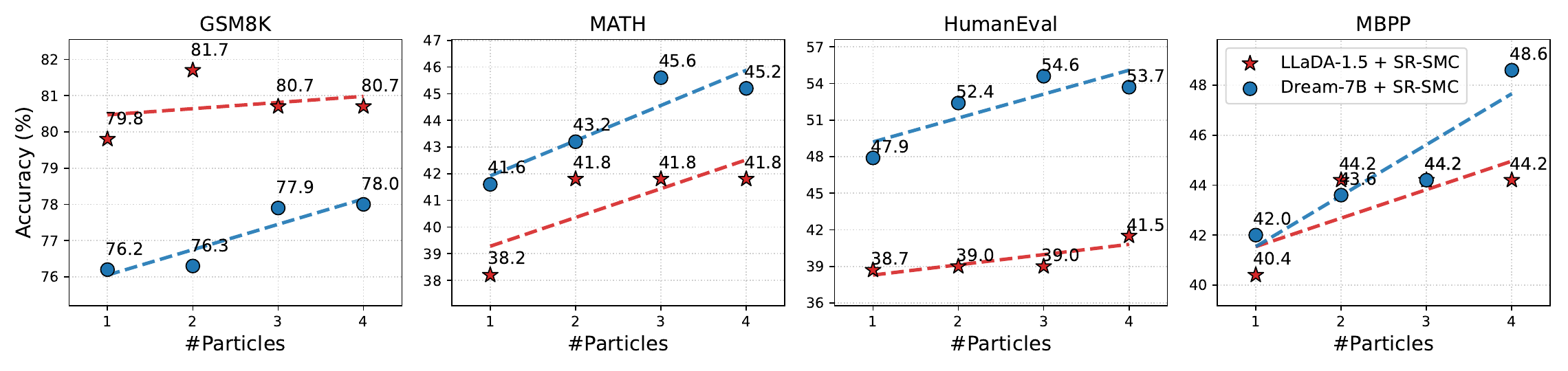}
    \caption{Comparison results of LLaDA 1.5 and Dream-7B using different numbers of particles on four tasks. Each marker denotes the empirical result while the dashed curves indicate first-order polynomial fits used solely to illustrate the overall trends as $N$ increases.}
    \label{fig:n_particles}
\end{figure}

\begin{figure}[t]
\setlength{\abovecaptionskip}{0.05in}
\setlength{\belowcaptionskip}{-0.1in}
  \centering
  \includegraphics[width=1.\linewidth]{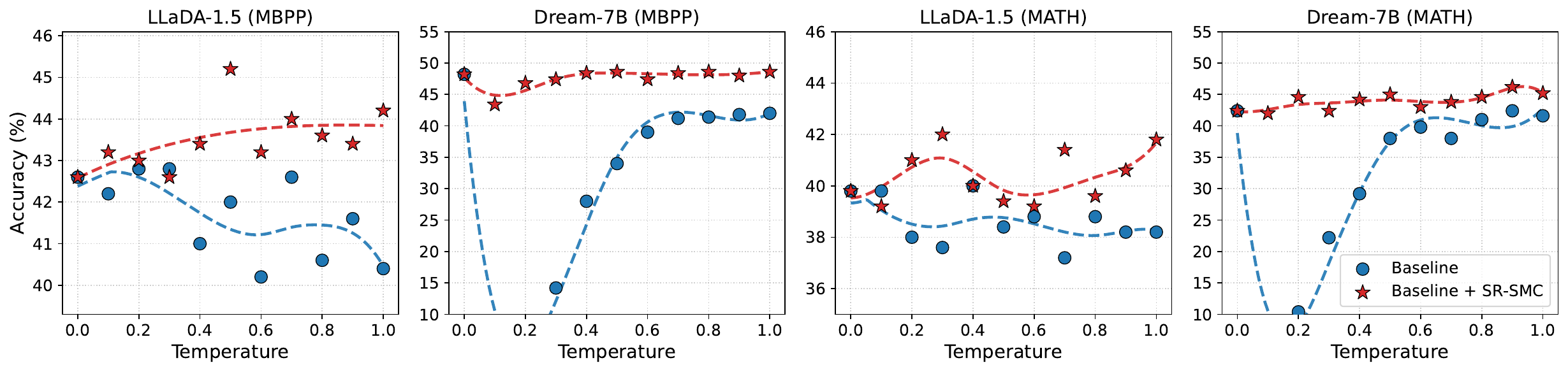}
  \caption{
  The accuracy of LLaDA-1.5 and Dream-7B on MBPP and MATH datasets as the temperature varies uniformly from 0.0 to 1.0. 
  The blue circles represent the baseline with standard parallel decoding, while the red stars denote the results using our SR-SMC with $N=4$ particles. 
  Our method consistently demonstrates better robustness across benchmarks and the entire temperature range.
  }
  \label{fig:ablation_temp}
\end{figure}

\subsection{Effect of Gumbel Noise}

In this study, we examine the robustness of SR-SMC under different sampling stochasticity by varying the Gumbel noise temperature $\tau \in [0,1.0]$. As shown in Figure~\ref{fig:ablation_temp}, standard parallel decoding is sensitive to the temperature, especially for Dream-7B, where low-temperature sampling leads to severe performance degradation on both MBPP and MATH, likely due to repetitive deterministic decoding. In contrast, SR-SMC with $N=4$ particles remains consistently more stable across the full temperature range. By maintaining multiple interacting trajectories and resampling them according to trajectory-level confidence, SR-SMC avoids being trapped by locally confident but suboptimal generations. For LLaDA-1.5, although the baseline is less unstable, SR-SMC still outperforms it under most temperature settings. These results demonstrate that SR-SMC improves not only inference-time scaling, but also the robustness of diffusion-based language generation to the choice of sampling temperature.

\subsection{Compute-Matched Comparison}
To verify that the gains of SR-SMC do not simply come from using more inference-time computation, we compare it with other inference-time scaling baselines under the same sampling budget. Specifically, we evaluate Best-of-$N$ and the stochastic remasking sampler ReMDM~\cite{wang2025remasking} with $4$K NFEs, matching the computational cost of MDLM with SR-SMC. 

As shown in Table~\ref{tab:compute_matched_mdlm}, all inference-time scaling methods improve over the MDLM baseline, indicating that additional sampling computation can generally benefit masked diffusion language models. The entropy remains reasonably high for all methods, showing that the improvement in sample quality is not caused by a trivial collapse toward overly deterministic generations. In addition, SR-SMC achieves the best generative perplexity among all compute-matched methods. This demonstrates that SR-SMC does not merely rely on drawing more samples or performing additional diffusion steps, but uses the extra computation more effectively through trajectory-level reweighting and resampling.

\begin{table}[t]
  \centering

  \begin{minipage}[t]{0.46\linewidth}
    \centering
    \caption{Compute matched comparison (Best-of-$N$ and ReMDM~\cite{wang2025remasking}) on text generation.}
    \label{tab:compute_matched_mdlm}
    \resizebox{\linewidth}{!}{
    \begin{tabular}{lcc}
      \toprule
      Method & Gen. PPL$\downarrow$ & Entropy$\uparrow$ \\
      \midrule 
      MDLM & 46.8 & 5.33  \\
      MDLM w/ Best-of-N & 38.6 & 5.31 \\
      MDLM w/ ReMDM & 35.5 & \textbf{5.35} \\
      MDLM w/ SR-SMC (Ours) & \textbf{25.8} & 5.15 \\
      \bottomrule
    \end{tabular}}
  \end{minipage}
  \hfill
  \begin{minipage}[t]{0.5\linewidth}
    \centering
    \caption{Ablation experiment of different SMC potentials on text generation.}
    \label{tab:potential_ablation_mdlm}
    \resizebox{\linewidth}{!}{
    \begin{tabular}{lcc}
      \toprule
      Method & Gen. PPL$\downarrow$ & Entropy$\uparrow$ \\
      \midrule 
      MDLM (baseline) & 46.8 & 5.33 \\
      SR-SMC (Ours) & \textbf{25.8} & 5.15 \\
      SR-SMC w/ token entropy & 39.0 & \textbf{5.36} \\
      SR-SMC w/ token prob. margin & 38.2 & 5.32 \\
      \bottomrule
    \end{tabular}}
  \end{minipage}

\end{table}

\subsection{Ablation on SMC Potentials}
We further study whether the proposed trajectory-level confidence is essential for the effectiveness of SR-SMC. To this end, we replace our trajectory-level confidence with two alternative self-rewarding signals: token entropy~\cite{ben2025accelerated} and probability margin~\cite{kim2025train}. The entropy-based potential encourages particles with lower predictive uncertainty, while the probability-margin potential favors particles whose selected tokens have a larger gap between the top prediction and competing alternatives. 

The results are reported in Table~\ref{tab:potential_ablation_mdlm}. Both alternatives improve over the MDLM baseline, reducing Gen. PPL from $46.8$ to $39.0$ and $38.2$, respectively. This confirms that the SMC framework can benefit from different model-internal confidence signals. Nevertheless, our trajectory-level probability potential achieves a substantially lower Gen. PPL of $25.8$, outperforming both alternatives by a large margin. These results indicate that accumulating the probabilities of accepted tokens along the generation trajectory provides a more effective global confidence estimate than step-wise uncertainty or local probability separation, supporting the design choice derived in Proposition~\ref{prop:confidence_weights}.

\begin{wrapfigure}{r}{0.44\columnwidth}
\setlength{\abovecaptionskip}{0.05in}
\setlength{\belowcaptionskip}{-0.2in}
\vspace{-0.18in}
    \centering
    \includegraphics[
        width=0.9\linewidth
    ]{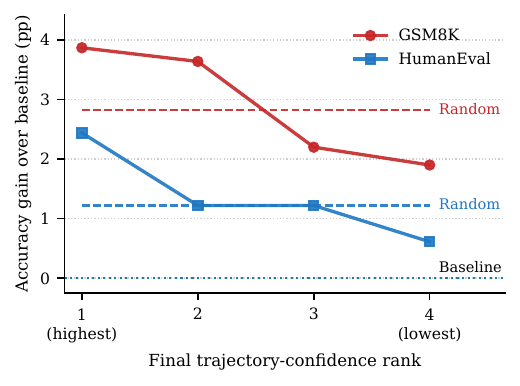}
    \caption{
    Confidence--correctness analysis on LLaDA-1.5 with $N=4$. Rank 1 denotes the highest-confidence particle, and dashed lines
    denote random selection.}
    \label{fig:confidence_correctness}
\end{wrapfigure}

\subsection{Trajectory Confidence Analysis}
\label{sec:confidence_correctness}

We examine whether trajectory confidence is informative for downstream correctness and whether maximum-weight selection is a reasonable final decoding rule. For each example, we rank the $N=4$ final particles according to their accumulated trajectory confidence and evaluate each rank separately. We additionally report uniform random selection from the final particle set.
As shown in Figure~\ref{fig:confidence_correctness}, the highest-confidence particle achieves the best performance on both GSM8K and HumanEval, and performance generally decreases with confidence rank. Maximum-confidence selection also outperforms random selection. These results indicate that trajectory confidence is positively associated with downstream correctness and provides a useful, although imperfect, criterion for final particle selection.

\subsection{Zero-Shot Evaluation}

\begin{wrapfigure}{r}{0.45\textwidth}
\setlength{\abovecaptionskip}{0.05in}
\setlength{\belowcaptionskip}{-0.2in}
\vspace{0.05in}
\captionof{table}{Zero-shot performance of LLaDA-1.5 and Dream-7B variants on GSM8K and MATH benchmarks with length $L=256$.}
    \label{tab:zero_shot_results}
    \resizebox{.99\linewidth}{!}{
    \begin{tabular}{ll|cc}
        \toprule
        Model & Benchmark & Baseline & w/ \textit{SR-SMC} \\
        \midrule
        \multirow{2}{*}{LLaDA-1.5}
        & GSM8K {\footnotesize (0-shot)} & 75.4 & \textbf{76.3} \\
        & MATH {\footnotesize (0-shot)}  & 35.2 & \textbf{36.6} \\
        \midrule
        \multirow{2}{*}{Dream-7B}
        & GSM8K {\footnotesize (0-shot)} & 77.9 & \textbf{81.1} \\
        & MATH {\footnotesize (0-shot)}  & 44.2 & \textbf{49.2} \\
        \bottomrule
    \end{tabular}}
\end{wrapfigure}

To evaluate generalizability across prompting settings, we conduct zero-shot experiments on GSM8K and MATH. We note that the results for HumanEval presented in our main experiments (Table~\ref{tab:benchmark_results}) are inherently zero-shot. For the MBPP benchmark, however, we observed that both base models failed to produce code in the requisite format for automated post-processing without few-shot examples. 
Unlike few-shot settings, zero-shot generation relies exclusively on the model's intrinsic reasoning capabilities and the effectiveness of the inference-time exploration. As illustrated in Table \ref{tab:zero_shot_results}, SR-SMC consistently enhances the performance of both LLaDA-1.5 and Dream-7B in these zero-shot settings. These results demonstrate that trajectory-level confidence serves as a robust implicit reward signal, enabling the model to effectively navigate a wider exploration space and steer away from low-quality outputs even in the absence of prompting demonstrations.

\section{Related Work}

As D3PM~\cite{austin2021structured} introduced the state absorbing with mask token for discrete diffusion models, masked diffusion language models have attracted increasing attention as a promising alternative to auto-regressive (AR) models~\cite{sahoo2024simple,lou2024discrete,schiff2024simple,shi2024simplified,arriola2025block}. Built upon MDLMs, diffusion large language models such as LLaDAs~\cite{nie2025large,zhu2025llada,bie2025llada2}, Dream~\cite{ye2025dream}, and DiffuLLaMA~\cite{gong2024scaling} have demonstrated strong scalability and achieved competitive performance when compared to similarly sized AR models. Subsequently, dLLM-Cache~\cite{liu2025dllm} and Fast-dLLM~\cite{wu2025fast} introduced caching and parallel decoding to MDLMs, further improving inference efficiency and their potential for real-world applications.

Despite these advances, existing MDLMs primarily improve performance through model scaling~\cite{nie2024scaling,nie2025large}, architectural modifications~\cite{wu2025fastv2,bie2025llada2}, or training-time interventions~\cite{schiff2024simple,hersche2025soft}, while the role of inference-time scaling remains largely unexplored. Recent works~\citep{ma2025inference,wu2023practical,dou2024diffusion,cardoso2023monte} proposed scaling test-time compute for diffusion models, substantially improving the performance beyond simply scaling diffusion sampling steps. \citet{singhal2025a} further illustrated this idea on MDLMs using the sequential Monte Carlo framework. \citet{dang2025inference} extend it with particle Gibbs sampling that enables generation refinement via an external task-specific reward guidance. In parallel, ReMDM~\cite{wang2025remasking} introduces a principled remasking strategy to improve text quality with more sampling steps. In contrast, our proposed algorithm is self-rewarding and can be used for general tasks with arbitrary pretrained models and remasking strategies.

\section{Conclusion}

In this paper, we propose a novel inference-time scaling algorithm for masked diffusion language models (MDLMs). Our key algorithm is a sequential Monte Carlo (SMC) method where the trajectory confidence is used as importance weights. This results in a self-rewarding SMC framework that promotes globally confident generation trajectories, without requiring additional training or external reward guidance. Extensive experiments and ablation studies across multiple benchmarks and model families demonstrate that our self-rewarding SMC significantly improves pretrained MDLMs in terms of both sample quality and diversity, unlocking an effective and principled inference-time scaling dimension for diffusion-based language generation.

\section*{Acknowledgements}

This research was partially supported by \emph{Kjell och M{\"a}rta Beijer Foundation} and by the project \emph{Blending probabilistic and nonlinear representations} (contract number: 2025-04318), funded by the Swedish Research Council.
Partial computational resources were provided by the National Academic Infrastructure for Supercomputing in Sweden (NAISS), funded by the Swedish Research Council.

\bibliography{main}
\bibliographystyle{plainnat}

\newpage

\appendix

\section{Sequential Monte Carlo for Diffusion Reverse Process}

\subsection{Incremental Importance Weights}
\label{app-subsec:smc_iiweight}

Consider a sequence of tokens $\rvx_T,\dots,\rvx_0$ in the diffusion reverse path, with unnormalized intermediate target distribution $\tilde{\pi}_t(\rvx_{t:T})$. We define the following Markov sequential proposal 
\begin{equation}
    q_t(\rvx_{t:T})=q_T(\rvx_T) \prod_{k=t}^{T-1} q_k(\rvx_k \mid \rvx_{t+1}),
\end{equation}
so that it satisfies a standard factorization:
\begin{equation}
    q_{t-1}(\rvx_{t-1:T}) = q_t(\rvx_{t:T})\, q_{t-1}(\rvx_{t-1}\mid \rvx_t).
\end{equation}
The importance weight of trajectory $\rvx_{t:T}$ (from the full-mask tokens to a partial masked state) at time $t$ is given by
\begin{equation}
    \tilde W_t(\rvx_{t:T}) = \frac{\tilde\pi_t(\rvx_{t:T})}{q_t(\rvx_{t:T})}.
    \label{app-eq:iseight}
\end{equation}
Likewise, at next step, we could write
\begin{align}
    \tilde W_{t-1}(\rvx_{t-1:T}) 
    &= \frac{\tilde\pi_{t-1}(\rvx_{t-1:T})}{q_{t-1}(\rvx_{t-1:T})} \\
    &= \frac{\tilde\pi_{t-1}(\rvx_{t-1:T})}{q_t(\rvx_{t:T}) \, q_{t-1}(\rvx_{t-1} \mid \rvx_t)} \\
    &= \underbrace{\frac{\tilde\pi_{t-1}(\rvx_{t-1:T})}{\tilde\pi_t(\rvx_{t:T}) \, q_{t-1}(\rvx_{t-1} \mid \rvx_t)}}_{=\tilde{w}_{t-1}(\rvx_{t-1:T})} \cdot \underbrace{\frac{\tilde\pi_t(\rvx_{t:T})}{q_t(\rvx_{t:T})}}_{=\tilde W_t(\rvx_{t:T})},
    \label{app-eq:smc-iiweight}
\end{align}
where $\tilde W_t(\rvx_{t:T})$ is the previous weight as in Eq.~\eqref{app-eq:iseight}, and $\tilde{w}_{t-1}(\rvx_{t-1:T})$ is the incremental importance weights in Eq.~\eqref{eq:smc_iiweights}. Conceptually, $\tilde{w}_{t-1}(\rvx_{t-1:T})$ is the local ratio that updates the global path-wise importance weight when extending trajectory with more unmasked tokens i.e., from $\rvx_{t:T}$ to $\rvx_{t-1:T}$. In practice, SMC maintains particle weights recursively via Eq.~\eqref{app-eq:smc-iiweight} and performs resampling using normalized version of $\tilde{W}_{t-1}(\rvx_{t:T})$.

\subsection{Proof for Confidence-based Sequential Monte Carlo}
\label{app-subsec:proof_prop31}

\textbf{Proposition 3.1.}
\textit{
Given a pretrained diffusion model $p_\theta$, let $\{\tilde\pi_t(\rvx_{t:T})\}_{t=0}^T$ denote the unnormalized path measures defined by the recursion in Eq.~\eqref{eq:fk_recursion}. If the sequential proposal in SMC is chosen to be the diffusion transition kernel, i.e., $q_{t-1}(\rvx_{t-1} \mid \rvx_t) = K_t(\rvx_t,\rvx_{t-1})$,
then the incremental importance weights at step $t-1$ are given by
\begin{equation}
    \tilde{w}_{t-1}(\rvx_{t-1:T}) = \prod_{j\in S_t} \rvc_t(j),
\end{equation}
where $\rvc_t (j) \coloneqq p_\theta\big(\hat{\rvx}_0(j) \mid \rvx_t \big)$ is the token confidence and $\gS_t$ denotes the selected mask subset to be updated at step $t$.
}

\begin{proof}
Recall that in sequential Monte Carlo, the incremental importance weight at step $t-1$ is defined as
\begin{equation}
    \tilde{w}_{t-1}(\rvx_{t-1:T}) = \frac{\tilde{\pi}_{t-1}(\rvx_{t-1:T})}{\tilde{\pi}_t(\rvx_{t:T}) \, q_{t-1}(\rvx_{t-1} \mid \rvx_t)},
    \label{app-eq:smc_iiweights}
\end{equation}
Substituting the path recursion in Eq.~\eqref{eq:fk_recursion} into the numerator yields
\begin{equation}
    \tilde{w}_{t-1}(\rvx_{t-1:T}) = \frac{\tilde{\pi}_t(\rvx_{t:T})\, K_t(\rvx_t,\rvx_{t-1})\, G_{t-1}(\rvx_t,\rvx_{t-1})}{\tilde{\pi}_t(\rvx_{t:T})\, q_{t-1}(\rvx_{t-1}\mid \rvx_t)}.
\end{equation}
Removing \(\tilde{\pi}_t(\rvx_{t:T})\) in both numerator and denominator, we obtain
\begin{equation}
    \tilde{w}_{t-1}(\rvx_{t-1:T}) =\frac{K_t(\rvx_t,\rvx_{t-1})\, G_{t-1}(\rvx_t,\rvx_{t-1})}{q_{t-1}(\rvx_{t-1}\mid \rvx_t)}.
\end{equation}
When the proposal distribution $q_{t-1}(\rvx_{t-1}\mid \rvx_t)$ is selected to be the same as the diffusion transition kernel, i.e, $q_{t-1}(\rvx_{t-1}\mid \rvx_t)=K_t(\rvx_t,\rvx_{t-1})$, the transition kernel disappears, and
\begin{equation}
    \tilde{w}_{t-1}(\rvx_{t-1:T}) = G_{t-1}(\rvx_t,\rvx_{t-1}).
\end{equation}
And recall that we define the potential as the joint probability of accepted tokens within set $\gS_t$, i.e., $G_{t-1}(\rvx_t,\rvx_{t-1}) = \prod_{j\in S_t} p_\theta(\rvx_{t-1}(j) \mid \rvx_t)$, and that $\rvc_t (j) \coloneqq p_\theta\big(\hat{\rvx}_0(j) \mid \rvx_t \big)$, we finally have the following
\begin{align}
    \tilde{w}_{t-1}(\rvx_{t-1:T}) 
    &= \prod_{j\in S_t} p_\theta(\rvx_{t-1}(j) \mid \rvx_t) \\
    &= \prod_{j\in S_t} \rvc_t(j),
\end{align}
which completes the proof. 
\end{proof}

\section{Limitation and Future Work}
\label{app-sec:limitation}

In general, the practical motivation of SR-SMC is consistent with prior work in LLMs: scaling test-time compute optimally can be more effective than simply scaling model parameters~\cite{snell2024scaling,liu2025can}. Within the dLLM setting, this is particularly relevant in quality-critical tasks such as code generation, offline generation, and benchmark evaluation, where users may prefer a slower but more reliable answer. Beyond direct inference, SR-SMC can also be used to generate higher-quality synthetic outputs or pseudo-targets from a pretrained dLLM, which may further benefit post-training, distillation, or preference optimization.

While the proposed self-rewarding SMC provides a principled framework for trajectory-level confidence-guided sampling, there are several worth-noting limitations. First, our method increases inference-time computation by running multiple diffusion processes in parallel. This trade-off is inherent to inference-time scaling methods and can be controllable by adjusting the number of particles. Second, the proposed trajectory confidence relies solely on model likelihood. While this choice is generic and task-agnostic, it does not explicitly optimize for downstream objectives such as reasoning correctness or human preference. In our future work, we plan to explore more informed proposals, such as look-ahead or twisted diffusion transitions, to further improve sampling efficiency and quality.

\begin{figure}[b]
    \centering
    \includegraphics[width=.8\linewidth]{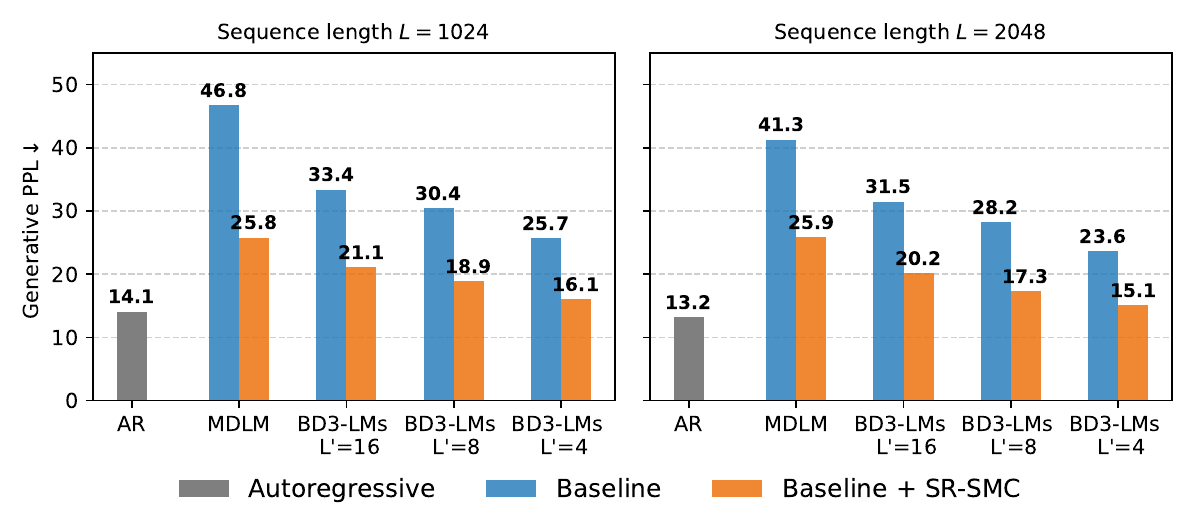}
    \caption{Generative perplexity ($\downarrow$) results on 300 samples of lengths $L=1024, 2048$ using our self-rewarding SMC and the corresponding baselines.}
    \label{app-fig:gen_ppl}
\end{figure}

\section{Additional Experiments}
\label{app-sec:experiments}

\subsection{Additional Results on Text Generation}
\label{app-subsec:entropy}

In Table~\ref{tab:gen_ppl} and Figure~\ref{app-fig:gen_ppl}, we primarily use generative perplexity to evaluate the sample quality of open-ended text generation. Here, we additionally report the entropy of generated samples in Table~\ref{app-tab:gen_ppl_entropy} to examine whether SR-SMC improves Gen. PPL by collapsing to low-diversity outputs. For reference, the Gen. PPL and entropy of the original data are $14.8$ and $5.44$, respectively, as reported in~\citet{wang2025remasking}.

Overall, SR-SMC substantially reduces Gen. PPL while keeping the entropy at a comparable level to the corresponding baselines and the original data. For example, on MDLM with $L=1024$, SR-SMC reduces Gen. PPL from $46.8$ to $25.8$, while the entropy only changes from $5.33$ to $5.15$. For BD3-LMs, SR-SMC consistently improves Gen. PPL across different block sizes, with entropy values remaining around $5.2$ for $L=1024$ and $5.5$ for $L=2048$. These values are close to the entropy of the original data and the AR baseline, suggesting that SR-SMC improves sample quality without causing mode collapse or overly deterministic generation.

\begin{table}[ht]
  \caption{Generative perplexity (Gen. PPL; $\downarrow$) and entropy results of 300 samples of lengths $L=1024, 2048$. All models are trained on OWT dataset. For reference, the Gen. PPL and entropy of the original data are 14.8 and 5.44, respectively, as reported in~\citet{wang2025remasking}.}
  \label{app-tab:gen_ppl_entropy}
  \centering
  \resizebox{.9\linewidth}{!}{
    \begin{tabular}{ll cccc}
      \toprule
      & & \multicolumn{2}{c}{$L=1024$} & \multicolumn{2}{c}{$L=2048$} \\
      \cmidrule(lr){3-4} \cmidrule(lr){5-6}
      Model & & Gen. PPL & Entropy & Gen. PPL & Entropy \\
      \midrule 
      AR & & 14.1 & 5.29  & 13.2 & 5.54  \\
      \midrule 
      MDLM & & 46.8 & 5.33  & 41.3 & 5.64  \\
      MDLM \textit{w/ Best-of-N} & & 38.6 & 5.31 & 40.3 & 5.61 \\
      MDLM \textit{w/ ReMDM} & & 35.5 & 5.35  & 45.1 & 5.27  \\
      MDLM \textit{w/ SR-SMC} & & 25.8 & 5.15 & 25.9 & 5.41 \\
      \midrule 
      \multirow{3}{*}{\makecell{BD3-LMs}}
      & \hspace{4.1em} $L'=16$ & 33.4 & 5.34 & 31.5 & 5.59 \\
      & \hspace{4.1em} $L'=8$ & 30.4 & 5.33 & 28.2 & 5.63  \\
      & \hspace{4.1em} $L'=4$ & 25.7 & 5.33 & 23.6 & 5.62  \\
      \hline
      \multirow{3}{*}{\makecell{BD3-LMs \textit{w/ SR-SMC}}}
      & \hspace{4.1em} $L'=16$ & 21.1 & 5.19 & 20.2 & 5.46  \\
      & \hspace{4.1em} $L'=8$ & 18.9 & 5.18 & 17.3 & 5.45  \\
      & \hspace{4.1em} $L'=4$ & 16.1 & 5.20 & 15.1 & 5.49  \\
      \bottomrule
    \end{tabular}}
\end{table}

\subsection{Detailed Results of Inference with Gumbel Noise}
\label{app-subsec:gumbel_noise}

We provide the detailed results of diffusion sampling with different Gumbel noise temperatures $\tau$ ranging from $0.0$ to $1.0$ in Table~\ref{app-tab:gumbel_tau}. We also note that the Fast-dLLM implementation may introduce small numerical perturbations under multi-GPU parallel decoding. Therefore, for the deterministic setting $\tau=0$, we treat SR-SMC and the corresponding baseline as using the same deterministic decoding result by default. This ensures that the comparison focuses on the effect of SR-SMC under stochastic sampling rather than implementation-level perturbations. Overall, our self-rewarding SMC (SR-SMC) consistently outperforms the baseline LLaDA-1.5~\cite{zhu2025llada} and Dream-7B~\cite{ye2025dream} over a wide range of temperatures on both MBPP and MATH benchmarks. Notably, the Dream-7B baseline is highly sensitive to the noise temperature, with performance degrading severely when $\tau$ is slightly increased to $0.1$ or $0.2$. In contrast, SR-SMC substantially improves robustness across different temperature settings, suggesting that particle-based exploration and trajectory-level resampling can mitigate the brittleness of standard diffusion sampling. 

\begin{table}[ht]
\centering
\small
\caption{Effect of Gumbel noise temperature $\tau$ on MBPP and MATH.}
\label{app-tab:gumbel_tau}
\resizebox{1.\linewidth}{!}{
\begin{tabular}{ll | ccccccccccc}
\toprule

Benchmark & Method & $\tau=0.0$ & $\tau=0.1$ & $\tau=0.2$ & $\tau=0.3$ & $\tau=0.4$ & $\tau=0.5$ & $\tau=0.6$ & $\tau=0.7$ & $\tau=0.8$ & $\tau=0.9$ & $\tau=1.0$ \\
\midrule

\multirow{4}{*}{\textbf{MBPP}}
& LLaDA-1.5  
& 42.6 & 42.2 & 42.8 & 42.8 & 41.0 & 42.0 & 40.2 & 42.6 & 40.6 & 41.6 & 40.4 \\
& LLaDA-1.5 \textit{w/ SR-SMC} 
& - & \textbf{43.2} & \textbf{43.0} & 42.6 & \textbf{43.4} & \textbf{45.2} & \textbf{43.2} & \textbf{44.0} & \textbf{43.6} & \textbf{43.4} & \textbf{44.2} \\

& Dream-7B 
& 48.2 & 1.80 & 5.60 & 14.2 & 28.0 & 34.0 & 39.0 & 41.2 & 41.4 & 41.8 & 42.0 \\
& Dream-7B \textit{w/ SR-SMC} 
& - & \textbf{43.4} & \textbf{46.8} & \textbf{47.4} & \textbf{48.4} & \textbf{48.6} & \textbf{47.4} & \textbf{48.4} & \textbf{48.6} & \textbf{48.0} & \textbf{48.6} \\

\midrule
\multirow{4}{*}{\textbf{MATH}}
& LLaDA-1.5  
& 39.8 & 39.8 & 38.0 & 37.6 & 40.0 & 38.4 & 38.8 & 37.2 & 38.8 & 38.2 & 38.2 \\
& LLaDA-1.5 \textit{w/ SR-SMC} 
& - & 39.2 & \textbf{41.0} & \textbf{42.0} & 40.0 & \textbf{39.4} & \textbf{39.2} & \textbf{41.4} & \textbf{39.6} & \textbf{40.6} & \textbf{41.8} \\

& Dream-7B  
& 42.4 & 2.40 & 10.4 & 22.2 & 29.2 & 38.0 & 39.8 & 38.0 & 41.0 & 42.4 & 41.6 \\
& Dream-7B \textit{w/ SR-SMC} 
& - & \textbf{42.0} & \textbf{44.6} & \textbf{42.4} & \textbf{44.2} & \textbf{45.0} & \textbf{43.0} & \textbf{43.8} & \textbf{44.6} & \textbf{46.2} & \textbf{45.2} \\

\bottomrule
\end{tabular}}
\end{table}

\subsection{Compute-Matched Comparison on Diffusion Large Language Models}
\label{app-sec:compute_matched}

We compare SR-SMC with compute-matched inference-time scaling baselines on Dream-7B, including Best-of-N, MCMC, and majority voting. More specifically, the Best-of-N algorithm generates N independent particles in parallel using the same diffusion denoising process as SR-SMC, but without resampling. The particle with the highest cumulative log-probability is selected as the final output. MCMC generates a complete sequence using the standard diffusion process. Then, for each refinement step, it randomly selects one block, re-masks it entirely, and regenerates it conditioned on the current content of all other blocks.
Note that all these additional compute-matched experiments are re-evaluated on 8 NVIDIA H200 GPUs for a fair comparison. As shown in Table~\ref{app-tab:compute_matched_dream}, Best-of-N, MCMC, and majority voting improve the baseline average score from $53.2$ to $55.5$, $53.9$, and $54.2$, respectively, while SR-SMC achieves the best average score of $57.5$. In particular, SR-SMC provides large gains on HumanEval and MATH, suggesting that trajectory-level reweighting and resampling are more effective than simply selecting from independent samples or aggregating final answers. We also compare SR-SMC with classifier-free guidance (CFG)~\cite{ho2022classifier} under a smaller compute budget using two particles. As shown in Table~\ref{app-tab:cfg_dream}, SR-SMC with two particles achieves an average score of $53.8$, outperforming CFG with guidance scale $1.5$, which obtains $52.3$. 

\begin{table}[ht]
  \caption{Compute-matched comparison on diffusion large language models Dream-7B.}
  \label{app-tab:compute_matched_dream}
  \centering
  \resizebox{.7\linewidth}{!}{
    \begin{tabular}{l cccc|c}
      \toprule
      Method & GSM8K & MATH & HumanEval & MBPP & Average  \\
      \midrule 
      Dream-7B (baseline)   & 77.3 & 42.2 & 50.3 & 43.2 & 53.2 \\
      Best-of-N             & 77.4 & 44.6 & 54.3 & 45.8 & 55.5  \\
      MCMC                  & 76.9 & 44.2 & 54.3 & 40.2 & 53.9  \\
      Majority vote         & 81.9 & 37.6 & 50.0 & 47.4 & 54.2  \\
      SR-SMC                & 78.9 & 45.4 & 60.7 & 45.2 & \textbf{57.5}  \\
      \bottomrule
    \end{tabular}}
\end{table}

\begin{table}[ht]
  \caption{Classifier-free guidance on Dream-7B.}
  \label{app-tab:cfg_dream}
  \centering
  \resizebox{.73\linewidth}{!}{
    \begin{tabular}{l cccc|c}
      \toprule
      Method & GSM8K & MATH & HumanEval & MBPP & Average  \\
      \midrule 
      CFG (1.5)             & 75.4 & 40.4 & 50.3 & 43.0 & 52.3 \\
      SR-SMC (2 particles)  & 76.3 & 43.6 & 54.3 & 41.0 & 53.8  \\
      \bottomrule
    \end{tabular}}
\end{table}

\subsection{Practical Efficiency}
\label{sec:practical_efficiency}

We further evaluate the practical inference cost of SR-SMC on
LLaDA-1.5. We compare standard decoding, Best-of-$N$, and SR-SMC under
identical hardware and decoding settings. Best-of-$N$ and SR-SMC both
use four particles, allowing us to isolate the additional overhead
introduced by trajectory-level reweighting and resampling. We report
the average wall-clock latency per example, generation throughput, and
the corresponding task accuracy.

As shown in Table~\ref{tab:practical_efficiency}, running four particles
naturally increases the inference cost relative to single-particle
decoding. However, compared with the compute-matched Best-of-$N$
baseline, SR-SMC introduces only a modest additional wall-clock
overhead while consistently achieving higher accuracy. At generation
length $256$, SR-SMC increases latency from $6.18$ to $6.31$ seconds on
GSM8K and from $5.32$ to $5.45$ seconds on HumanEval, while improving
accuracy from $81.8$ to $82.2$ and from $40.2$ to $41.7$, respectively.
Similar trends are observed at generation length $512$. These results
indicate that most of the inference cost arises from evaluating multiple
particles, whereas the additional SMC operations provide measurable
performance gains with relatively small wall-clock overhead.

\begin{table}[ht]
  \centering
  \small
  \setlength{\tabcolsep}{4pt}
  \caption{
  Quality-efficiency comparison of LLaDA-1.5.
  Best-of-$N$ and SR-SMC use the same four-particle budget.
  Latency is measured under identical hardware and
  decoding settings.
  }
  \label{tab:practical_efficiency}
  \resizebox{0.8\textwidth}{!}{
  \begin{tabular}{lccccc|ccc}
    \toprule
    & & \multicolumn{3}{c}{GSM8K}
    & \multicolumn{3}{c}{HumanEval} \\
    \cmidrule(lr){3-5}
    \cmidrule(lr){6-8}
    Method
    & Length
    & Latency (s)$\downarrow$
    & Tok/s$\uparrow$
    & Acc.$\uparrow$
    & Latency (s)$\downarrow$
    & Tok/s$\uparrow$
    & Acc.$\uparrow$ \\
    \midrule

    Baseline
    & 256
    & 1.96
    & 107.81
    & 79.5
    & 2.81
    & 31.33
    & 37.7 \\

    Best-of-$N$
    & 256
    & 6.18
    & 34.46
    & 81.8
    & 5.32
    & 15.52
    & 40.2 \\

    SR-SMC
    & 256
    & 6.31
    & 32.84
    & \textbf{82.2}
    & 5.45
    & 14.46
    & \textbf{41.7} \\

    \midrule

    Baseline
    & 512
    & 2.97
    & 69.31
    & 80.6
    & 4.42
    & 20.32
    & 33.8 \\

    Best-of-$N$
    & 512
    & 10.50
    & 18.01
    & 80.7
    & 13.08
    & 10.83
    & 39.0 \\

    SR-SMC
    & 512
    & 11.36
    & 15.85
    & \textbf{82.5}
    & 13.79
    & 7.18
    & \textbf{40.7} \\

    \bottomrule
  \end{tabular}
  }
\end{table}

\subsection{Additional Zero-Shot Results}
\label{app-subsec:zero_shot}

In the main text, we report zero-shot results with generation length $L=256$. Here, we provide the complete zero-shot evaluation with both $L=256$ and $L=512$ in Table~\ref{app-tab:zero_shot_results}. Overall, SR-SMC consistently improves both LLaDA-1.5 and Dream-7B across GSM8K and MATH under different generation lengths. For $L=512$, SR-SMC improves the average score of LLaDA-1.5 from $59.7$ to $63.4$, and Dream-7B from $63.4$ to $64.8$. These results further confirm that the benefit of trajectory-level confidence remains stable under longer generation lengths and zero-shot prompting.

\begin{table}[ht]
    \centering
    \caption{Zero-shot performance of LLaDA-1.5 and Dream-7B variants on GSM8K and MATH benchmarks across different generation lengths.}
    \label{app-tab:zero_shot_results}
    \resizebox{.7\linewidth}{!}{
    \begin{tabular}{lc|cc|cc}
        \toprule
        Benchmark & Length & LLaDA-1.5 & \textit{w/ SR-SMC} & Dream-7B & \textit{w/ SR-SMC} \\
        \midrule
        GSM8K & 256 & 75.4 & \textbf{76.3} & 77.9 & \textbf{81.1} \\
        {\footnotesize (0-shot)} & 512 & 80.4 & \textbf{82.0} & 78.0 & 78.0 \\
        \midrule
        MATH & 256 & 35.2 & \textbf{36.6} & 44.2 & \textbf{49.2} \\
        {\footnotesize (0-shot)} & 512 & 39.0 & \textbf{44.8} & 48.8 & \textbf{51.6} \\
        \bottomrule
    \end{tabular}}
\end{table}

\subsection{Exploring Beyond Confidence-Based Framework}
\label{app-subsec:klass}

All dLLM experiments in the main text are built on top of the Fast-dLLM framework, as described in Section~\ref{sec:experiment}. Thus, our reported results already show that SR-SMC can be combined with a stronger decoding framework using KV cache and parallel decoding, rather than only with naive confidence-based sampling. To further examine its compatibility with more advanced dLLM samplers, we additionally apply SR-SMC on top of KL-Guided fast inference (KLASS)~\cite{kim2025klass} using LLaDA-1.5. As shown in Table~\ref{tab:klass}, KLASS + SR-SMC improves the average score from $49.84$ to $51.01$, with gains on GSM8K, MATH, and HumanEval, and no degradation on MBPP. These results suggest that SR-SMC is a general particle-based inference-time scaling method that can be integrated with stronger dLLM decoding frameworks.

\begin{table}[ht]
  \caption{Applying SR-SMC to KLASS framework.}
  \label{tab:klass}
  \centering
  \resizebox{.7\linewidth}{!}{
    \begin{tabular}{l cccc|c}
      \toprule
      Method & GSM8K & MATH & HumanEval & MBPP & Average \\
      \midrule 
      KLASS         & 77.48 & 34.2 & 40.24  & 47.47 & 49.85 \\
      KLASS + SR-SMC & 78.92 & 35.0 & 42.68 & 47.47 & 51.02 \\
      \bottomrule
    \end{tabular}}
\end{table}

\subsection{ESS Analysis}

We further analyze the effective sample size (ESS) of SR-SMC to examine whether particle degeneracy occurs during sampling. In our implementation, adaptive resampling is checked once per block rather than at every decoding step. This allows each particle to follow its own trajectory over a longer horizon, making the accumulated trajectory-level confidence more informative for resampling than purely step-wise token confidence.
Table~\ref{tab:ess} reports the final ESS values using $N=4$ particles across different tasks. The ESS values range from $2.65$ to $3.52$, remaining well above $1$, which suggests that the final particle population does not collapse to a single dominant particle. These results indicate that SR-SMC can improve generation quality while still maintaining a reasonable level of particle diversity.

\begin{table}[ht]
  \caption{The final ESS values using SR-SMC (LLaDA-1.5 with 4 particles) across tasks.}
  \label{tab:ess}
  \centering
  \resizebox{.55\linewidth}{!}{
    \begin{tabular}{l ccccc}
      \toprule
      Task & GSM8K & MATH & HumanEval & MBPP  \\
      \midrule 
       ESS value & 2.79 & 3.52 & 2.65  & 3.05  \\
      \bottomrule
    \end{tabular}}
\end{table}

\subsection{Additional Examples}
\label{app-sec:add_examples}

We provide additional qualitative examples comparing greedy decoding with SR-SMC in 
Figure~\ref{fig:additional_case_studies}. These examples illustrate that greedy decoding can commit to locally confident but incorrect intermediate steps, leading to persistent reasoning or calculation errors. In contrast, SR-SMC maintains multiple candidate trajectories and uses trajectory-level confidence for resampling, which helps recover more globally coherent solutions.

\begin{figure}[ht]
    \centering
    \resizebox{0.9\linewidth}{!}{
    \begin{minipage}{\linewidth}

    \begin{tcolorbox}[colback=gray!2, colframe=gray!50, arc=2pt, boxrule=0.5pt, title=\textbf{Qualitative Comparison: Arithmetic Reasoning}]
        \small
        \textbf{Problem:} Simplify $\frac{1}{5} \cdot \frac{8}{7} \div \frac{12}{20}$. \\
        \textbf{Reference Answer:} $\frac{8}{21}$
        
        \vspace{0.12cm}
        \begin{minipage}[t]{0.48\linewidth}
            \begin{tcolorbox}[colback=red!5, colframe=red!30, title=\textbf{Greedy Decoding ($\tau=0$)}, height=8.2cm]
                \scriptsize
                First, simplify the multiplication of the fractions: \\
                \[ \frac{1}{5} \cdot \frac{8}{7} = \frac{8}{7} \] 
                \textcolor{red}{(Calculation Hallucination: ignored $\frac{1}{5}$)} \\

                Next, simplify the division of the fractions:
                \[\frac{8}{7} \div \frac{12}{20}=\frac{8}{7} \cdot \frac{20}{12}\]
                Now, multiply the numerators and denominators:
                \[\frac{8 \cdot 20}{7 \cdot 12} = \frac{160}{84}\]
                Simplify the fraction:
                \[\frac{160 \div 40}{84 \div 40} = \frac{4}{7}\]

                \textbf{Final Answer: $\boxed{\frac{4}{7}}$} \textcolor{red}{\ding{55}}
            \end{tcolorbox}
        \end{minipage}
        \hfill
        \begin{minipage}[t]{0.48\linewidth}
            \begin{tcolorbox}[colback=green!10, colframe=green!40!black!30, title=\textbf{w/ SR-SMC ($\tau=1, N=4$)}, height=8.2cm]
                \scriptsize
                To simplify $\frac{1}{5} \cdot \frac{8}{7} \div \frac{12}{20}$:

                1. \textbf{Multiply the fractions first:}
                \[\frac{1}{5} \cdot \frac{8}{7}=\frac{1 \cdot 8}{5 \cdot 7}=\frac{8}{35}\]

                2. \textbf{Divide the fractions:}
                \[\frac{8}{35} \div \frac{12}{20}=\frac{8}{35} \cdot \frac{20}{12}=\frac{160}{420}\]

                3. \textbf{Simplify the fraction:}
                \[\frac{160}{420} = \frac{8}{21}\]

                \textbf{Final Answer: $\boxed{\frac{8}{21}}$} \textcolor{green!70!black}{\ding{51}}
            \end{tcolorbox}
        \end{minipage}
    \end{tcolorbox}


    \begin{tcolorbox}[colback=gray!2, colframe=gray!50, arc=2pt, boxrule=0.5pt, title=\textbf{Qualitative Comparison: Distance Calculation}]
        \small
        \textbf{Problem:} At 50 miles per hour, how far would a car travel in $2\frac{3}{4}$ hours? Express your answer as a mixed number. \\
        \textbf{Reference Answer:} $137\frac{1}{2}$
        
        \vspace{0.12cm}
        \begin{minipage}[t]{0.48\linewidth}
            \begin{tcolorbox}[colback=red!5, colframe=red!30, title=\textbf{Greedy Decoding ($\tau=0$)}, height=8.2cm]
                \scriptsize
                To determine how far a car would travel in $2\frac{3}{4}$ hours at 50 miles per hour, use:
                \[\text{Distance} = \text{Speed} \times \text{Time}\]
                First, convert the mixed number:
                \[2\frac{3}{4} = \frac{1}{4}\]
                \textcolor{red}{(Severe Conversion Error)}

                Next, multiply speed by time:
                \[\text{Distance} = 50 \times \frac{1}{4} = 12.5\]

                Finally, convert to a mixed number:
                \[12.5 = 12\frac{1}{2}\]

                \textbf{Final Answer: $\boxed{12\frac{1}{2}}$} \textcolor{red}{\ding{55}}
            \end{tcolorbox}
        \end{minipage}
        \hfill
        \begin{minipage}[t]{0.48\linewidth}
            \begin{tcolorbox}[colback=green!10, colframe=green!40!black!30, title=\textbf{w/ SR-SMC ($\tau=1, N=4$)}, height=8.2cm]
                \scriptsize
                To determine how far a car would travel in $2\frac{3}{4}$ hours at 50 miles per hour, use:
                \[\text{Distance} = \text{Speed} \times \text{Time}\]

                Convert the mixed number:
                \[2\frac{3}{4}=\frac{8}{4}+\frac{3}{4}=\frac{11}{4}\]

                Multiply speed by time:
                \[\text{Distance}=50 \times \frac{11}{4}=\frac{550}{4}=137.5\]

                Convert to a mixed number:
                \[137.5 = 137\frac{1}{2}\]

                \textbf{Final Answer: $\boxed{137\frac{1}{2}}$} \textcolor{green!70!black}{\ding{51}}
            \end{tcolorbox}
        \end{minipage}
    \end{tcolorbox}

    \end{minipage}
    }
    \caption{Additional qualitative comparisons of reasoning trajectories. Greedy decoding can commit to locally confident but incorrect intermediate steps, while SR-SMC maintains multiple candidate trajectories and recovers globally coherent solutions.}
    \label{fig:additional_case_studies}
\end{figure}


\newpage
\section*{NeurIPS Paper Checklist}

\begin{enumerate}

\item {\bf Claims}
    \item[] Question: Do the main claims made in the abstract and introduction accurately reflect the paper's contributions and scope?
    \item[] Answer: \answerYes{} 
    \item[] Justification: See the abstract and introduction sections.

\item {\bf Limitations}
    \item[] Question: Does the paper discuss the limitations of the work performed by the authors?
    \item[] Answer: \answerYes{} 
    \item[] Justification: See appendix~\ref{app-sec:limitation}.

\item {\bf Theory assumptions and proofs}
    \item[] Question: For each theoretical result, does the paper provide the full set of assumptions and a complete (and correct) proof?
    \item[] Answer: \answerYes{} 
    \item[] Justification: See the method section and appendix.

    \item {\bf Experimental result reproducibility}
    \item[] Question: Does the paper fully disclose all the information needed to reproduce the main experimental results of the paper to the extent that it affects the main claims and/or conclusions of the paper (regardless of whether the code and data are provided or not)?
    \item[] Answer: \answerYes{} 
    \item[] Justification: See experiments and the code in supplementary materials.

\item {\bf Open access to data and code}
    \item[] Question: Does the paper provide open access to the data and code, with sufficient instructions to faithfully reproduce the main experimental results, as described in supplemental material?
    \item[] Answer: \answerYes{} 
    \item[] Justification: The code has been uploaded in the supplementary materials and will be released publicly.

\item {\bf Experimental setting/details}
    \item[] Question: Does the paper specify all the training and test details (e.g., data splits, hyperparameters, how they were chosen, type of optimizer) necessary to understand the results?
    \item[] Answer: \answerYes{} 
    \item[] Justification: See experiments and appendix. Our method consistently demonstrates performance improvements compared to baselines.

\item {\bf Experiment statistical significance}
    \item[] Question: Does the paper report error bars suitably and correctly defined or other appropriate information about the statistical significance of the experiments?
    \item[] Answer: \answerNo{} 
    \item[] Justification: All our experiments are done once.

\item {\bf Experiments compute resources}
    \item[] Question: For each experiment, does the paper provide sufficient information on the computer resources (type of compute workers, memory, time of execution) needed to reproduce the experiments?
    \item[] Answer: \answerYes{} 
    \item[] Justification: See the experiment section.
    
\item {\bf Code of ethics}
    \item[] Question: Does the research conducted in the paper conform, in every respect, with the NeurIPS Code of Ethics \url{https://neurips.cc/public/EthicsGuidelines}?
    \item[] Answer: \answerYes{} 
    \item[] Justification: We have reviewed the NeurIPS Code of Ethics.

\item {\bf Broader impacts}
    \item[] Question: Does the paper discuss both potential positive societal impacts and negative societal impacts of the work performed?
    \item[] Answer: \answerNA{} 
    \item[] Justification: The proposed method only changes the sampling procedure of existing models.
    
\item {\bf Safeguards}
    \item[] Question: Does the paper describe safeguards that have been put in place for responsible release of data or models that have a high risk for misuse (e.g., pre-trained language models, image generators, or scraped datasets)?
    \item[] Answer: \answerNo{} 
    \item[] Justification: No such risk.

\item {\bf Licenses for existing assets}
    \item[] Question: Are the creators or original owners of assets (e.g., code, data, models), used in the paper, properly credited and are the license and terms of use explicitly mentioned and properly respected?
    \item[] Answer: \answerYes{} 
    \item[] Justification: Data and models are properly cited in our paper.

\item {\bf New assets}
    \item[] Question: Are new assets introduced in the paper well documented and is the documentation provided alongside the assets?
    \item[] Answer: \answerYes{} 
    \item[] Justification: \answerNA{}

\item {\bf Crowdsourcing and research with human subjects}
    \item[] Question: For crowdsourcing experiments and research with human subjects, does the paper include the full text of instructions given to participants and screenshots, if applicable, as well as details about compensation (if any)? 
    \item[] Answer: \answerNA{} 
    \item[] Justification: not involved in this paper.

\item {\bf Institutional review board (IRB) approvals or equivalent for research with human subjects}
    \item[] Question: Does the paper describe potential risks incurred by study participants, whether such risks were disclosed to the subjects, and whether Institutional Review Board (IRB) approvals (or an equivalent approval/review based on the requirements of your country or institution) were obtained?
    \item[] Answer: \answerNA{} 
    \item[] Justification: \answerNA{}

\item {\bf Declaration of LLM usage}
    \item[] Question: Does the paper describe the usage of LLMs if it is an important, original, or non-standard component of the core methods in this research? Note that if the LLM is used only for writing, editing, or formatting purposes and does \emph{not} impact the core methodology, scientific rigor, or originality of the research, declaration is not required.
    \item[] Answer: \answerNA{} 
    \item[] Justification: We didn't use LLMs for any core component.

\end{enumerate}

\end{document}